%% file: main-arxiv.tex
\documentclass{article}
\usepackage[utf8]{inputenc}

\title{Improved Upper Bounds for Slicing the Hypercube}

\usepackage{microtype}
\usepackage{amsmath}
\usepackage{amssymb}
\usepackage{amsthm}
\usepackage{enumitem}
\usepackage{graphicx}
\usepackage{multicol}
\usepackage{booktabs}
\usepackage{xcolor}
\usepackage{hyperref}
\usepackage{cleveref}
\usepackage{algorithm, algorithmicx, algpseudocode}  
\usepackage{makecell}
\usepackage{authblk}
\usepackage{xurl}
\usepackage{natbib}
\usepackage{mathtools}
\usepackage{placeins}
\newtheorem{definition}{Definition}
\newtheorem{theorem}{Theorem}

\newtheorem{corollary}{Corollary}

\DeclareMathOperator*{\argmax}{argmax}

\makeatletter
\renewcommand\AB@affilsepx{~ \protect\Affilfont}
\makeatother
\newcommand{\primarycontrib}{$^{*}$}
\author[1]{Duncan Soiffer\thanks{Equal contribution. Corresponding authors: Duncan Soiffer (\href{mailto:dsoiffer@cs.cmu.edu}{dsoiffer@cs.cmu.edu}), Nathaniel Itty (\href{mailto:nathanielitty@gmail.com}{nathanielitty@gmail.com}), Christopher D. Rosin (\href{mailto:christopher.rosin@gmail.com}{christopher.rosin@gmail.com}), Daniel Reichman (\href{mailto:dreichman@wpi.edu}{dreichman@wpi.edu})}}
\author[2]{Nathaniel Itty\primarycontrib}
\author[4]{Christopher D. Rosin\primarycontrib}
\author[2]{Blake Bruell}
\author[2]{Mason DiCicco}
\author[3]{G\'abor N. S\'ark\"ozy}
\author[2]{Ryan Offstein}
\author[3]{Daniel Reichman}
\affil[1]{Carnegie Mellon University}
\affil[2]{Work done while at Worcester Polytechnic Institute}
\affil[3]{Worcester Polytechnic Institute}
\affil[4]{Constructive Codes}
\renewcommand\Affilfont{\small}

\begin{document}

\date{February 20, 2026}
\maketitle

\begin{abstract}
A collection of hyperplanes $\mathcal{H}$ slices all edges of the $n$-dimensional hypercube $Q_n$ with vertex set $\{-1,1\}^n$ if, for every edge $e$ in the hypercube, there exists a hyperplane in $\mathcal{H}$ intersecting $e$ in its interior. Let $S(n)$ be the minimum number of hyperplanes needed to slice $Q_n$. We prove that
$S(n) \leq \lceil \frac{4n}{5} \rceil$, except when $n$ is an odd multiple of $5$, in which case $S(n) \leq \frac{4n}{5} + 1$. This improves upon the previously known upper bound of $S(n) \leq \lceil\frac{5n}{6}\rceil$ due to Paterson reported in 1971. We also obtain new lower bounds on the maximum number of edges in $Q_n$ that can be sliced using $k < n$ hyperplanes. We prove the improved upper bound on $S(n)$ by constructing $8$ hyperplanes slicing $Q_{10}$, aided by the recently introduced CPro1: an automatic tool
that uses reasoning LLMs coupled with automated hyperparameter tuning to create search algorithms for the discovery of mathematical constructions.
\end{abstract}

\input{sections-arxiv/introduction}

\input{sections-arxiv/preliminaries-and-results}
\input{sections-arxiv/related-work}

\input{sections-arxiv/algorithm}
\input{sections-arxiv/conclusion}
\input{sections-arxiv/acknowledgements}

\bibliographystyle{plain}
\bibliography{reference} 

\newpage
\appendix
\input{sections-arxiv/appendix}

\end{document}

%% file: sections-arxiv/introduction.tex
\section{Introduction}
Finding the value of $S(n)$, the number of hyperplanes required to slice all edges of the $n$-dimensional hypercube, has been recognized as a major open problem in discrete geometry with 
%application to other fields, such as 
applications in other areas, such as
combinatorics~\cite{sauermann2025improved,alon2002balancing}, perceptrons~\cite{o1971hyperplane}, and discrete neural networks (threshold circuits). For example, Paturi and Saks have shown that $S(n)$ is related to the design of efficient threshold circuits computing the parity function~\cite{paturi1990threshold}. 
%in circuit complexity, threshold circuits are studied as a model of the computational capabilities of neural networks, where Paturi and Saks show that in any depth-2 circuit for parity, the set of hyperplanes associated with the threshold gates at the first level of the circuit must intersect every edge of the $n$-dimensional hypercube \cite{paturi1990threshold,saks1993slicing}.

Despite significant attention, there is still a large gap between upper and lower bounds for $S(n)$. While recent progress on the lower bound has been made, with the best known lower bound to date~\cite{sauermann2025improved} as $\Omega(n^{13/19}\log^{-32/12}n)$, the upper bound has not been improved in more than fifty years.
It is easy to see $S(n) \leq n$; for example, take $n$ axis-parallel hyperplanes $x_i=0$. Although it may seem that $S(n)=n$, a construction due to Paterson reported by O'Neil~\cite{o1971hyperplane} in 1971 established that $S(n) \leq \lceil\frac{5n}{6} \rceil$, which has remained the best known upper bound for $S(n)$.  

To approach this problem, we develop local search algorithms to find collections of $k$ hyperplanes that slice all edges of the $n$-dimensional hypercube $Q_n$ for values of $k<\lceil\frac{5n}{6}\rceil$. We also construct smaller collections of hyperplanes which slice many edges of the $n$-dimensional hypercube for various values of $n$ and $k$. To create these algorithms, we utilize the recently introduced CPro1 \cite{rosin2025neurips}, a large language model (LLM)-driven system for the automated creation of search algorithms for combinatorial design problems. Supporting this, we formalize the concept of the reduced hypercube, which projects the $n$-dimensional hypercube onto an $\ell$-dimensional lattice corresponding to a composition of $n$ hyperplane coefficients into $\ell$ equal groups. Enforcing such a composition drastically reduces the search space and enables more efficient edge-slicing calculation, and reflects a naturally occurring pattern in the collections of hyperplanes which slice many edges of $Q_n$ we observed. We leverage this concept and CPro1's ability to produce a high volume of candidate algorithms, and draw insights by observing partial solution features in order to manually design a final algorithm which obtains a solution that slices $Q_{10}$ with $8$ hyperplanes.

% The remarkably structured nature of the discovered solutions led us to formalize the concept of the reduced hypercube. %, which may help drive further understanding of the hypercube slicing problem.
% We define this construction by enforcing a composition of the $n$ dimensions into $\ell$ groups of equal coefficients. Under this constraint, we can determine whether an edge is sliced by a hyperplane using the aggregate sum of coordinates within each group. This grouping projects the $n$-dimensional hypercube onto an $\ell$-dimensional integer lattice, drastically reducing the size of the search space.

This construction improves the decades-old upper bound and suggests further improvements are within reach. 
Further, we believe the remarkably structured nature of the discovered solutions is of interest, and may help drive further understanding of the hypercube slicing problem.
%The use of CPro1, which is open source and freely available, in the new constructions suggests it also has potential use for future constructions of mathematical significance. % This is more like something to put in the conclusion I feel
Beyond the contribution to the slicing problem, our new bounds may be useful for evaluating computer programs directed toward finding mathematical constructions, and our process of arriving at these bounds illustrates the importance of incorporating human oversight even in high-volume AI-driven research and discovery systems. The use of LLMs was limited to program generation during the high-volume search process, results from which exposed patterns that were interpreted and synthesized by human reviewers, and they were not involved in mathematical derivations, experiment design and orchestration, or writing. The tools used in this work are open source, and we make our solutions, experimental configurations, and source code publicly accessible at \url{https://github.com/DSoiffer/upper-bounds-for-hypercube-slicing}, providing a foundation for future work on this problem and program-oriented automated mathematical discovery.

% Moved this to the declarations
% Code, experimental configurations, and supporting documentation is available at \href{https://github.com/DSoiffer/upper-bounds-for-hypercube-slicing}{github.com/DSoiffer/upper-bounds-for-hypercube-slicing}. 

%...and our [process?] illustrates the importance of including humans at multiple points throughout the automated discovery process.

% maybe add something to the effect "we hope the remarkably structured nature of the discovered solutions will help shed light on..."

%% file: sections-arxiv/preliminaries-and-results.tex
\section{Preliminaries and main result}
\label{sec:preliminaries}
The $n$-hypercube $Q_n=(V_n,E_n)$ is an undirected graph with $V_n=\{-1,1\}^n$ and $E_n$ the set of all pairs of vertices from $V_n$ differing by exactly one coordinate ($|V_n|=2^n, |E_n|=n2^{n-1}$). 
We identify each edge $(v_1,v_2) \in E_n$ with the line segment in $\mathbb{R}^n$ connecting $v_1$ to $v_2$.

A \emph{hyperplane} $H$ in $\mathbb{R}^n$ is the set of all solutions to the linear equation $\langle a,x\rangle=b$ for $a \in \mathbb{R}^n$ and $b \in \mathbb{R}$. Namely, $H=\{x \in \mathbb{R}^n \mid \langle a,x \rangle=b\}$. 
We say that a hyperplane $H$ \emph{slices} an edge $(v_1,v_2)$ of $Q_n$ if it intersects the line segment connecting $v_1$ to $v_2$ in its interior. In other words, $\langle a,v_1\rangle-b$ and $\langle a,v_2\rangle-b$ have opposite signs: $(\langle a,v_1\rangle-b)\cdot(\langle a,v_2\rangle-b) <0$. We call a collection of hyperplanes $\mathcal{H}$ a \emph{full solution} when every edge in $Q_n$ is sliced by at least one hyperplane in $\mathcal{H}$; otherwise, we call $\mathcal{H}$ a \emph{partial solution}.

Let $S(n)$ be the minimum number that ensures the existence of a collection $\mathcal{H}$ of $S(n)$ hyperplanes such that $\mathcal{H}$ is a full solution.  
% Let $S(n)$ be the minimum number that ensures the existence of a collection $\mathcal{H}$ of $S(n)$ hyperplanes such that every edge in $Q_n$ is sliced by at least one hyperplane in $\mathcal{H}$.  
%It is easy to see $S(n) \leq n$, for example, take $n$ axis-parallel hyperplanes $x_i=0$. Although it may seem that $S(n)=n$, a construction due to Paterson reported in O'Neil~\cite{o1971hyperplane} established that $S(n) \leq \lceil\frac{5n}{6} \rceil$, which has remained the best known upper bound for $S(n)$.  Our main contribution is improving upon this upper bound: 
Our main contribution is improving on the $S(n) \leq \lceil\frac{5n}{6} \rceil$ bound established by Paterson \cite{o1971hyperplane}:

\begin{theorem}\label{main-theorem}
For every $n$ such that $n$ is not an odd multiple of $5$, $S(n) \leq \lceil\frac{4n}{5} \rceil$. Otherwise, $S(n) \leq \frac{4n}{5}+1$.
\end{theorem}

\begin{proof}
It is well known (e.g., \cite{saks1993slicing}) that \(S(n)\) is subadditive: for any \(k,\ell \in \mathbb{N}\),
\[
S(k+\ell) \leq S(k) + S(\ell).
\]
We claim that if the hypercube \(Q_{10}\) can be sliced using \(8\) hyperplanes (that is, \(S(10) \leq 8\)), this implies that
\[
S(n) \leq \left\lceil \frac{4n}{5} \right\rceil
\]
for every $n$, except when $n$ is an odd multiple of $5$ where it implies $S(n)\leq \frac{4n}{5} +1$. We proceed by cases.

For $n \leq 5$, it is known \cite{newbounds} that $S(n)=n$.

For $6 \leq n < 10$, subadditivity together with the known bound $S(6) \leq 5$ implies that $S(n) \leq n - 1$, which is equivalent to $S(n) \leq \left\lceil \frac{4n}{5} \right\rceil$ in this range.

For $10 \leq n \leq 15$, the bounds follow directly from subadditivity and the assumption \(S(10) \leq 8\).

For $16 \leq n < 20$, the bound is obtained from the inequality $S(16 + j) \leq S(10) + S(6) + j$ for $j \in \{0,\dots,3\}$.

Finally, for $n \geq 20$, the inequalities follow by repeated applications of subadditivity in a manner analogous to the cases for $10 \leq n < 20$.

To complete the proof, we present $8$ hyperplanes that slice $Q_{10}$. These are listed below, where each row corresponds to the vector of coefficients of a hyperplane, with the final entry representing the bias term.
 For example, the first hyperplane is $-2x_1-2x_2-2x_3-2x_4-2x_5-2x_6+1x_7+3x_8-8x_9-1x_{10}=0.5$.
\begin{equation*}
\label{eq:108planes}
\begin{aligned}
\begin{array}{rrrrrrrrrrr}
-2 & -2 & -2 & -2 & -2 & -2 &  1 &  3 & -8 & -1 & 0.5 \\
-2 & -2 & -2 & -2 & -2 & -2 & -1 & -3 &  8 &  1 & 0.5 \\
-2 & -2 & -2 & -2 & -2 & -2 & -1 &  8 &  3 & -1 & 0.5 \\
-2 & -2 & -2 & -2 & -2 & -2 &  1 & -8 & -3 &  1 & 0.5 \\
-2 & -2 & -2 & -2 & -2 & -2 &  4 & -1 &  1 & -7 & 0.5 \\
-2 & -2 & -2 & -2 & -2 & -2 & -4 &  1 & -1 &  7 & 0.5 \\
-2 & -2 & -2 & -2 & -2 & -2 & -7 & -1 & -1 & -4 & 0.5 \\
-2 & -2 & -2 & -2 & -2 & -2 &  7 &  1 &  1 &  4 & 0.5
\end{array}
\end{aligned}
\end{equation*}

It can be verified (\Cref{appendix:verification-algorithm}) that these planes slice all 5120 edges of $Q_{10}$, completing the proof.
\end{proof}

This construction of $8$ slicing hyperplanes is not unique, and additional examples are provided in \Cref{appendix:additional-constructions-and-observations}. 

Let $S(n,k)$ be the maximum number of edges of $Q_n$ that can be sliced using $k$ hyperplanes. We additionally improve on the best lower bounds known for various $S(n,k)$~\cite{nowack2022slicing,emamy2008coverings,newbounds} and provide new bounds for values of $n,k$ that have not been studied previously. These are summarized in \Cref{tab:main1} and \Cref{tab:main2}.

% {
% \setlength{\tabcolsep}{5pt} %slightly thinner margins so that it fits
\begin{table}[h]
\centering
\caption{Lower bounds achieved on $S(n, k)$ up to $n=11$. Bold indicates new best bounds, underline indicates a tight lower bound (full slicing).}\label{tab:main1}
\begin{tabular}{|c|ccccccc|}
\hline
$S(n,k)$ & $k=4$ & $k=5$ & $k=6$ & $k=7$ & $k=8$ & $k=9$ & $k=10$  \\
\hline
$n=5$  & \underline{78}   & \underline{80}   & --   & --   & --   & --   & --       \\
$n=6$  & 184  & \underline{192}  & \underline{192}  & --   & --   & --   & --        \\
$n=7$  & 410  & \textbf{440}  & \underline{448}  & \underline{448}  & --   & --   & --        \\
$n=8$  & \textbf{920}  & 980  & \textbf{1018} & \underline{1024} & \underline{1024} & --   & --        \\
$n=9$  & \textbf{1974} & \textbf{2184} & \textbf{2266} & \textbf{2298} & \underline{2304} & \underline{2304} & --        \\
$n=10$ & \textbf{4312} & \textbf{4704} & \textbf{4998} & \textbf{5088} & \textbf{\underline{5120}} & \underline{5120} & \underline{5120}     \\
$n=11$ & \textbf{9072} & \textbf{10248} & \textbf{10816} & \textbf{11128} & \textbf{11240} & \textbf{\underline{11264}} & \underline{11264}  \\
\hline
\end{tabular}
\end{table}
% }

\begin{table}[h]
\centering
\caption{Lower bounds achieved on $S(n, k)$ for $k\leq9\leq12$ and $n=k+3$. Parenthetical numbers indicate the total edges in $E_n$, bold indicates new best bounds.}\label{tab:main2}
\begin{tabular}{|c|cccc|}
\hline
$S(n,k)$ & $k=9$ & $k=10$ & $k=11$ & $k=12$ \\
\hline
% $n=k+3$ & \textbf{24552} & \textbf{53224}   & \textbf{114666} & \textbf{245748} \\
$n=k+3$
& \makecell{\textbf{24552} \\ (24576)}
& \makecell{\textbf{53224} \\ (53248)}
& \makecell{\textbf{114666} \\ (114688)}
& \makecell{\textbf{245748} \\ (245760)} \\
\hline
\end{tabular}
\end{table}

%TODO: do we say LLM here before introducing the abbreviation?
To discover these constructions, we make use of the LLM-driven CPro1 protocol~\cite{rosin2025using}, combining its ability to autonomously generate and evaluate diverse candidate algorithms with observations on the highly structured nature of the resulting solutions to develop an edge-weighted hill-climbing algorithm capable of finding collections of planes which slice many edges of $Q_n$.

%% file: sections-arxiv/related-work.tex
\section{Related work}
%TODO: should we have explicit subsections here?
The determination of $S(n)$ has been studied extensively in multiple fields such as combinatorics~\cite{sauermann2025improved,alon2002balancing}, geometry~\cite{grunbaum1972cut}, machine learning~\cite{o1971hyperplane}, and Boolean circuit complexity~\cite{paturi1990threshold}. More information on the connection of $S(n)$ to these fields can be found in the comprehensive survey by Saks~\cite{saks1993slicing}.

O'Neil~\cite{o1971hyperplane} proved that $ S(n)=\Omega(\sqrt{n})$. This lower bound has seen recent improvements~\cite{yehuda2021slicing,klein2023slicing,sauermann2025improved} starting with a 
$\Omega(n^{0.51})$ lower bound~\cite{yehuda2021slicing}. The best known lower bound to date~\cite{sauermann2025improved} is 
$\Omega(n^{13/19}\log^{-32/12}n)$. It is conjectured (see e.g., ~\cite{sauermann2025improved}) that $S(n)=\Omega(n)$. So far such a linear lower bound is only known under additional restrictions on the slicing hyperplanes such as having non-negative coefficients~\cite{ahlswede1990identity,gotsman1994spectral} or coefficients in $\{-1,1\}$~\cite{alon2002balancing}, where a lower bound of $n/2$ is known. Recently, Sauermann and Xu~\cite{sauermann2025nondegenerate} proved a lower bound of $n/(4C)$ on the number of hyperplanes needed to slice $Q_n$ when the coefficients of the hyperplanes are integers in $\{-C,\ldots,C\}$ for a fixed integer $C>0$.

Significant mathematical and computational efforts have been directed to improve the upper bound on $S(n)$~\cite{emamy1986cuts,emamy2021cut,ziegler2000computing,emamy2008coverings}.
Roughly speaking, a recurring idea in the computational works is to compute all sliceable subsets of edges of $Q_n$: sets of edges that can simultaneously be
cut by a single hyperplane.
%(one way to achieve this is by guessing which endpoint of each edge belongs to one of the two regions created by the hyperplane and then use linear programming to find a separating hyperplane). %is this actually how all possibe sliceable sets are found? that doesn't seem quite right to me?
Thereafter, solve the resulting Set Cover instance with parameter $k$ to decide whether $k$ sets from the set family consisting of all sliceable sets can cover all edges of $Q_n$. 
These works resulted in upper and lower bounds on $S(n,k)$, and established that $S(n)=n$ for all $n \leq 5$ and $S(6)=5$; it is currently open if $S(7)=5$. For a recent survey of these results, see~\cite{nowack2022slicing}.
Despite these efforts, the Paterson bound has not been improved since it was announced in~\cite{o1971hyperplane}.
 
A central challenge in algorithmically finding slicing hyperplanes is the exponential growth of the vertex and edge sets of $Q_n$ and the exponential complexity of solving the NP-hard Set Cover problem. %TODO: maybe should be a bit more precise that finding/enumerating the maximum sliceable sets is the real issue here (I think)
These factors combine to produce prohibitive running times even for relatively small values of $n$. The curse of dimensionality presents an additional obstacle: an exhaustive search over even a single hyperplane with integer coefficients in $[-C, C]$ quickly becomes infeasible for modest values of $C$ and $n$; similarly, determining all sliceable subsets of $E_n$ quickly becomes infeasible.
In contrast, we rely on local search algorithms, which avoid the prohibitive running time arising from previous approaches to compute $S(n)$.
%but do not provide impossibility guarantees.

The use of AI as a tool in the process of proving or disproving mathematical statements has increased in recent years~\cite{romera2024mathematical,gladkov2025bunkbed,charton2024patternboost,georgiev2025mathematical,bubeck2025early,rosin2025using,berczi2026flow}.
In particular, our use of generative AI for the discovery of algorithms that can find objects of mathematical significance is influenced by recent breakthroughs achieved using methods such as FunSearch \cite{romera2024mathematical}, its successor AlphaEvolve \cite{novikov2025alphaevolve}, and CPro1 \cite{rosin2025neurips}. 

CPro1 \cite{rosin2025neurips} and AlphaEvolve \cite{novikov2025alphaevolve} are both systems which revolve around using large language models (LLMs) to generate large numbers of programs in service of solving mathematical problems.
CPro1 has solved long-standing open instances of combinatorial design problems \cite{rosin2025a} and instances described in recent research literature \cite{rosin2025using}. AlphaEvolve has been used to generate new algorithms, optimize existing programs, and improve bounds and constructions for an impressive range of mathematical problems. %TODO: i think i will try to give AlphaEvolve a bit more credit here, feels a little bit understated
The primary difference between CPro1 and AlphaEvolve lies in how they generate candidate programs: CPro1 uses an LLM to propose diverse strategies and implement them from scratch, and provides automated hyperparameter tuning for programs that expose hyperparameters, whereas AlphaEvolve starts from an initial program and iteratively refines it using LLMs in a sophisticated evolutionary loop.
While AlphaEvolve is not available for public use, the open source OpenEvolve \cite{openevolve} and CodeEvolve \cite{assumpccao2025codeevolve} tools have replicated or improved on several of AlphaEvolve's results, and have been used to produce new results \cite{barbarians,scalinglaws}. Recently, Aletheia \cite{aletheia} was announced, which is a similarly autonomous LLM-based mathematics research system, but whose focus lies towards proof generation rather than computational search leading to verifiable mathematical constructions as we do here.  %problems which can be approached computationally.

The hypercube slicing problem differs from some of the problems that have recently attracted interest in studies of generative AI and other computational approaches for mathematics, such as circle and hexagon packing and minimizing ratios of maximum to minimum distance~\cite{novikov2025alphaevolve,georgiev2025mathematical,assumpccao2025codeevolve,parczyk2025new}. In these problems, one is interested in finding a construction minimizing or maximizing a certain real parameter (for example, placing $n$ disjoint unit hexagons in a larger hexagon whose side length is as small as possible). In contrast, slicing the hypercube involves a ``hard" constraint (either all edges are sliced by a set of hyperplanes or they are not). 
%unlike constructions involving the maximization of a continuous parameter. 
% \dun{I don't really think this whole paragraph is a meaningful distinction, since I am pretty sure plenty of the other problems AI was used in were optimizing discrete variables. (Technically, we only succeed if we find a full slicing, but we can easily just frame it as maximizing the cut number).
% }

A prominent theme when considering AI assisted mathematical proofs is the extent to which the proof depends on AI tools such as LLMs. Using the recent informal taxonomy of quantifying the degree of reliance of math proofs on AI tools~\cite{tao_mathstodon_2025}, our proof falls somewhere between ``AI-powered modifications of existing solutions (which could be either human-generated or AI-generated)" to ``Complex interactions between humans and AI tools in which the AI tools provided crucial calculations." Under the taxonomy provided in \cite{aletheia}, our proof falls under ``Level C: A substantive human-AI collaboration where both parties contribute in an essential way."

%% file: sections-arxiv/algorithm.tex
\section{Algorithms for Slicing the Hypercube}\label{Sec:Algorithms}
We found the collection of hyperplanes presented in Theorem \ref{main-theorem} by developing a local search algorithm targeted toward the hypercube slicing problem.
Initially, we manually created an algorithm (\Cref{alg:tabu_search}) based on tabu search, which has previously proven useful in constructing mathematical objects~\cite{parczyk2025new} and optimization methods geared towards combinatoric problems~\cite{mehrabian2023finding}.
After many rounds of alteration and optimization, this algorithm found constructions slicing all but $6$ edges of $Q_{10}$ with $8$ planes, but fell short of improving the upper bound $S(n) \leq \lceil \frac{5n}{6} \rceil$ established by Paterson. 
After a prolonged period without improved results, we opted to explore algorithmic approaches to the hypercube slicing problem using CPro1 \cite{rosin2025neurips}, with the results from our tabu algorithm serving as high-quality test cases.

\subsection{Algorithm Search} \label{sec:algorithm-discovery}

CPro1 prompts a large language model with a formal problem definition and asks it to propose a diverse set of solution strategies, which it then implements as efficient C programs that are autonomously evaluated in a sandboxed environment.
To use CPro1, problems are defined by a set of constraints, and a specific \textit{instance} of that problem assigns concrete values to the parameters of those constraints; the task is to determine whether the constraints are satisfiable for that instance. In hypercube slicing, the constraint is to give a collection of $k$ hyperplanes which slices $s$ edges of $Q_n$. An instance specifies specific values for the tuple $(n, k, s)$, with $s=|E_n|$ to require a full solution. %as well as additional hyperparameters the LLM may create for its algorithm and choose to expose for automated tuning. 
%In hypercube slicing, an instance specifies the hypercube dimension $n$, the number of hyperplanes $k$, and the minimum number of edges $s$ that must be sliced (with $s=|E_n|$ to require a full solution).

CPro1 evaluates candidate programs on a collection of ``dev" instances known to be solvable---here, instances $(n,k,s)$ for which our tabu search had already discovered solutions. The candidates which slice the most edges on these instances most quickly are then evaluated on instances for which no solution is currently known.

We configured CPro1 with a definition of the hypercube slicing problem with integer coefficients restricted to $[-40,40]$ and no other constraints. We performed five runs of CPro1, with 200 candidate programs generated per run. Two runs used small dev instances with $n<8$, while the remaining three used larger dev instances with $n$ up to $17$. CPro1 generated and tested diverse strategies, such as simulated annealing and tabu search, and variations within these strategies, such as alternative cost functions.

From the runs that used small dev instances, the best program used an adaptively edge-weighted hillclimbing search (\Cref{sec:algorithms-local-search}). The method was effective for small instances: it solved all our dev instances for $n<8$, matching the results of our tabu search. However, this technique by itself could not find a set of $8$ hyperplanes slicing $Q_{10}$.
From the runs that used larger dev instances, several programs were able to improve upon our previous best results for $n>12$ with $k=n-3$, though still fell short of slicing all edges. One such run produced an improvement for $n=15,$ $k=12$ where all coefficients were equal to the same value in the first 13 dimensions across all planes, and only the last 2 columns varied. 
% \dr{An example we want to add to the Supplementary information (SI)}.\cnote{done}
While this program was not more generally successful, and could not find a full solution for $n=10,$ $k=8$, we decided to pursue this intriguing constraint ourselves.

We manually modified the adaptively edge-weighted hill-climbing search to constrain the initial columns to all have the same constant value. Observing bias terms with low magnitudes across many high-performing partial solutions, and informed by observations from manually constructing the tabu search algorithm, we also fixed all bias terms to $0.5$, an arbitrary offset to avoid passing through vertices. Following this, we manually explored varying values for the constrained columns and for the number of constrained columns, improving upon our best partial results for instances including $n=11$, $k=8$, with further runs obtaining the full solution to $n=10,$ $k=8$.

Following this, we explored other approaches to improving the algorithm. By hill-climbing in smaller local neighborhoods, we reduced the number of times computation is spent checking proposed moves which are later rejected. Further, observing that the planes in solutions for $n=10$, $k=8$, and $n=6$, $k=5$ slice very similar numbers of edges in the reduced hypercube induced by their fixed columns (\Cref{sec:constrained}), we added a variance penalty which encourages planes to slice similar numbers of edges in the reduced hypercube. While these approaches yielded improved partial results and significant performance gains---decreasing the time to find solutions for $n=10$, $k=8$ by more than two orders of magnitude when combined---they proved insufficient for finding full solutions in dimensions $n>10$ for $k=n-3$. 

The AI method CPro1 provided the ingredients here in separate solutions: the efficient edge-weighted hillclimbing search, and the constraint on initial columns, but did not by itself produce a combined program that could solve $n=10$ $k=8$. In hindsight, it seems the insight on constraining the initial columns was close at hand: our previous experiments with tabu search had already constrained several values in each row to be the same value, just not the same columns or the same value across all rows, and we also observe now that Paterson's original solution to $n=6$ $k=5$ can be put in a form in which the first 3 columns are all equal to the same identical constant while only the last three columns vary. This highlights how AI-based research and discovery tools can produce large volumes of potential insights in the process of exploration, but also highlights the importance of surfacing these results for human review and synthesis.

%TODO: For PNAS, adapt paragraph above to be on the interplay between human and AI, human in the loop, 

\subsection{Local search}\label{sec:algorithms-local-search}

Across all proposed programs, we found that local hill-climbing search algorithms were consistently most effective. Such algorithms begin with an initial candidate solution (i.e., $k$ hyperplanes), then:

\begin{enumerate}
    \item Given the current candidate, generate similar candidates---say, by changing a single coefficient of a hyperplane---replacing the current if the ``fitness'' improves (in our case, a count of the edges sliced).
    \item Repeat until a fitness threshold is reached.
\end{enumerate}

Out of many such local search strategies, we found that approaches with adaptive weighting yielded dramatically improved solutions. This involves incrementing the ``weight'' of each edge---the value contributed to the fitness function when sliced---if it remains unsliced for many iterations. Such dynamic weighting is a technique commonly used in other combinatorial optimization problems (e.g. \cite{numvc}) to help algorithms escape local optima.

We define a set of allowed neighbors to search over in step (1). We allow neighbors $\mathcal H'$ with Hamming distance $\leq2$ within a single plane, that is, we search over collections of hyperplanes where exactly one hyperplane differs in $\leq2$ coordinates. We further restrict this by requiring that each coordinate differs by no more than $d$, a hyperparameter, which reduces computation spent evaluating moves which are unlikely to be productive. Searches are always performed over integer coefficients; if a full solution $\mathcal{H}^*$ with $k$ planes exists, then there exists a full solution with $k$ planes and integral coefficients \citep{sauermann2025nondegenerate}.

\subsection{Reduced Hypercube}
\label{sec:constrained}
Following our results from tabu search and initial testing with CPro1, we observed that across several runs with strong slicing results, every hyperplane obeyed a composition (a representation of $n$ as $\ell$ positive integers, $n=\sum_{i=1}^{\ell}b_i$) %which essentially prescribes groups of identical coefficients. 
%Formally, given a hypercube dimension $n$, we define a pattern by an ordered list of positive integers $[b_1,\ldots,b_\ell]$ with sum $n$. 
where each $b_i$ denotes the size of a group of equal coefficients. For example, the hyperplanes in the solution described in \Cref{main-theorem} %\eqref{eq:108planes}
satisfy a $[6,1,1,1,1]$ composition.
Enforcing a particular composition not only drastically reduces the dimensionality of the search space, but also enables a natural optimization to the fitness computation: rather than operating on the $n \times 2^{n-1}$ edges of the $n$-cube, we map the cube to a grid in $\ell$ dimensions with extent $[b_1,\ldots,b_\ell]$.
%we project the cube to a cell complex in $\ell$ dimensions with extent $[b_1,\ldots,b_\ell]$. 

We formalize this reduction with the concept of a \emph{reduced hypercube} (proofs and further details in \Cref{appendix:reduced-hypercube}). Given a composition $[b_1,\ldots,b_\ell]$ of $n$, any vertex $(v_1, v_2,\cdots, v_n)$ has a reduced representation with $\ell$ coordinates, where each new coordinate is the sum of the coordinates in the composition. 
This grid has only
\[
|V|=\prod_{i=1}^\ell(b_i+1)\] vertices, where two reduced vertices are connected by an edge if and only if they differ by exactly 2 in one coordinate, corresponding to changing the sign of a coordinate in the original hypercube. Furthermore, this grid has
\[|V|\left(\ell-\sum_{i=1}^\ell\frac{1}{{b}_i +1}\right)\]
edges, 
%This grid only has \sum_{i=1}^\ell b_i \prod_{j\neq i} (b_j + 1)$ edges\dr{Are we referring to the number of vertices? If yes, the first sum is not needed and the product should be from $i$ to $\ell$. If no, an explanation should be added for this formula},
and computing the fitness of a candidate hyperplane reduces to evaluating slicing of these (multiplicity weighted) %\dr{What are the weights}
grid edges. %Note that we evaluate the slicing of the ``reduced'' version of the hyperplane in $\ell$ dimensions where each coefficient is drawn from its corresponding block.
Given a collection of planes $\mathcal H$ satisfying such a composition, we let $\phi(\mathcal H)$ denote the function which yields the set of reduced edges.

As an example, enforcing $[6,1,1,1,1]$ maps $Q_{10}$ to a $5$-dimensional, $6 \times 1 \times 1 \times 1 \times 1$ grid via the linear projection $(x_1,\ldots,x_{10}) \to (x_1 + \cdots + x_6,x_7,\ldots,x_{10})$. This grid has $320$ edges -- a $16 \times$ reduction from the $5120$ edges of $Q_{10}$.
In addition, the slicing computations only involve $5$ multiplications per hyperplane instead of $10$.

When searching over hyperplanes under these constraints, we find that using $|\phi(\mathcal H)|$ as a fitness function is more effective for finding full solutions, while a version which retains edge multiplicities corresponding to the unreduced hypercube is more effective at finding high-slicing partial solutions.
Further, we observe that search algorithms typically perform best under highly constrained patterns with only a single $b_i>1$ and all other $b_j=1$, so long as $b_i$ is not too large. For example, $b=[6,1,1,1,1]$ proves to be an effective composition for $n=10$.
%It is curious that solutions are possible in such a constrained space.

We also observe that while the number of edges sliced by each plane in a full solution $\mathcal H^*$ can vary significantly, the number of \emph{reduced} edges sliced is nearly constant across each plane. %That is, $|\phi(H_i)| \approx |\phi(H_j)| ~~ \forall H_i, H_j \in \mathcal H^*$.
For instance, for every full solution we found for $Q_{10}$, $\left| |\phi(H_i)| - |\phi(H_j)| \right| \leq 4 ~~ \forall H_i, H_j \in \mathcal H^*$.
Inspired by this observation, we subtract a variance penalty that encourages planes to slice similar numbers of edges in the reduced hypercube.
Formally, given a weighting $w$ on each edge, a set of $k$ hyperplanes $\mathcal{H}$, and letting $\mu=\sum_{H \in \mathcal H} \frac{|\phi(H)|}{k}$, we define fitness function $\psi$ as
\begin{equation*}
    % \begin{split}
    \psi(\mathcal H, w) \coloneq 
    \sum_{(v_1,v_2) \in \phi(\mathcal H)} \big[ w(v_1,v_2) \big] - %\operatorname{Var}\bigl(\{\,|\phi(H)| : H\in\mathcal H\,\}\bigr)
    \sum_{H \in \mathcal H} \frac{(|\phi(H)|-\mu)^2}{k}
    % \end{split}
\end{equation*}

Together, these alterations to the basic adaptive edge-weighted search framework yield \Cref{alg:full_algorithm}.

%TODO: probably swap psi and phi, linguistically phi makes more sense for fitness
% Also: biases are set to 0 (I think this should also be mentioned as an insight we got from Tabu search)

\begin{algorithm}[t]
\caption{Our adaptive edge-weighted search}\label{alg:full_algorithm}
\begin{algorithmic}[1]
\Require composition $b$; the edges of the corresponding reduced hypercube for $n$ dimensions $E_n$; number of planes $k$; hyperparameters max\_iterations, weight\_period, and weight\_limit %and weight limit $L$
%TODO: does it make sense to have an \Ensure here for H*?
\State $\mathcal H^* \gets \emptyset$
\While{time limit not exceeded}
\State $\mathcal H \gets \text{random collection of $k$ planes satisfying $b$}$ 
% \State Initialize weights $w:E_n\to \mathbb{R}$ to $ w(v_1,v_2) := 1$ for all $(v_1,v_2) \in E_n$
\State Initialize weights $w:E_n\to \mathbb{R}$ to $1 ~\forall(v_1,v_2) \in E_n$

\While{$t < \text{max\_iterations}$}
    \State $\mathcal H' \gets \operatorname{RandomNeighbor}(\mathcal{H}, b)$ 

    \If{$\psi(\mathcal H', w) \ge \psi(\mathcal H, w)$} \Comment{Accept if $\geq$ weighted fitness} 
        \If{$\psi(\mathcal H', w) > \psi(\mathcal H, w)$} \Comment{Reset timer if $>$ weighted fitness}
            \State $t_w \gets 0$ %\Comment{Reset timer if $>$ fitness}
        \EndIf
        \If {$|\phi(\mathcal H')| > |\phi(\mathcal H)|$ } \Comment{Reset timer if higher slicing number} 
            \State $t \gets 0$ %\Comment{Reset timer if slices more} 
        \EndIf
        \State $\mathcal{H} \gets \mathcal H'$ \Comment{Accept new solution} 
    \EndIf

    \If{$t_w > \text{weight\_period}$} \Comment{Increase weights for unsliced edges}
        % \For{$(v_1, v_2)\notin \phi(\mathcal H)$} 
        %     \State $w(v_1,v_2) \gets \min\left(w(v_1,v_2)+1, L\right)$
        % \EndFor
        \State $w(v_1,v_2) \gets \min\left(w(v_1,v_2)+1, \text{weight\_limit}\right) ~~\forall (v_1, v_2) \notin \phi(\mathcal H)$
        \State $t_w \gets 0$
    \EndIf

    \State $t_w \gets t_w+1$
    \State $t \gets t+1$

\EndWhile
\State $\mathcal H^* \gets \argmax_{h \in \{\mathcal{H}^*,\mathcal H\}} |\phi(h)|$ \Comment{Track overall best solution}
% \If{$|\phi(\mathcal H)| > |\phi(\mathcal H^*)|$} \Comment{Track overall best solution}
%     \State $\mathcal H^* \gets \mathcal H$ 
% \EndIf
\EndWhile

\end{algorithmic}
\end{algorithm}

\subsection{AI-Driven Discovery and Researcher Intervention}
Is the use of generative AI for our improved constructions necessary? Prior to involving AI, we experimented extensively for over a year with tabu search without improving the upper bound on $S(n)$. Certainly, it is possible, and perhaps probable, that with further time and computational investment we would have discovered hyperplane arrangements which slice $Q_{10}$. %(and tabu search has been used before for finding combinatorial constructions~\cite{parczyk2025new}).
However, utilizing an AI-driven automated discovery system rapidly accelerated this progress by dramatically increasing the volume of ideas that could be attempted and from which insights could be gleaned.
 %TODO: need to phrase this better, add more

It is also natural to ask: Is human involvement necessary to arrive at our solutions? %We demonstrate that only by combining human insight with the high-volume discovery potential of automated research 
While LLM-driven automated systems are powerful enough to discover solutions that slice $Q_6$ with $5$ planes and attain strong partial slicing results, they were unable to obtain improvements on $S(n)$ without further human intervention. Only by combining researcher oversight with the insights surfaced by automated search and discovery were we able to achieve the results in Tables \ref{tab:main1} and \ref{tab:main2}.

%To examine these questions,
To support our assessment, we present results from OpenEvolve, CPro1, and our initial tabu search algorithm when ran on the hypercube slicing problem.
For OpenEvolve, we evaluate using the LLMs OpenAI o4-mini, OpenAI GPT-5-mini, Google Gemini 3 Pro and Flash, and use the provided templates for generating Rust code; for CPro1 we evaluate with OpenAI o4-mini and generate C code (\Cref{sec:algorithm-discovery}). Both systems are given a similar number of attempts (programs). We observe the different methods of generating programs between OpenEvolve and CPro1 reflected in their LLM usage patterns: over 80\% of CPro1's LLM tokens are \textit{output} tokens from generating diverse candidate programs starting from the problem definition, whereas over 80\% of OpenEvolve's LLM tokens are \textit{input} tokens for analysis of prior programs and results. 
These systems are still new, and it is not yet clear what the right patterns are for using LLMs to generate programs for solving mathematical problems. Some studies have found AlphaEvolve's evolutionary loop unnecessary for some of the results it has achieved \cite{randomsampling}, but its success across a wide variety of mathematical and non-mathematical domains demonstrates the framework's versatility.

%TODO: write this section better (not just a list)
We compare results obtained from our manually created tabu search, OpenEvolve without further instruction beyond the problem description, OpenEvolve with the additional constraint that hyperplanes should contain groups of identical coefficients, CPro1 before human intervention, and our final results with CPro1 after introducing our manual changes. \Cref{tab:alg-comparison} demonstrates that OpenEvolve and CPro1 without further oversight could achieve strong results, reproducing the bound on $S(n)$ due to Paterson and in some cases improving on our previous partial slicing constructions achieved by tabu search, but failed to produce a new upper bound for $S(n)$.
Meanwhile, CPro1 code with our manually introduced changes consistently achieves state of the art results, highlighting the importance of introducing human oversight at points within the discovery process beyond initial problem formulation.

\begin{table}
\centering
\caption{Comparison of the number of edges sliced by our manually created tabu search algorithm, AI-only systems, and AI systems with human intervention. Underlines indicate full hypercube slicing.} 
\begin{tabular}{|l|c|c|c|c|c|}
\hline
 & \makecell{Tabu \\search} & \makecell{OpenEvolve} & \makecell{Composition-\\constrained \\ OpenEvolve} & \makecell{CPro1} & \makecell{CPro1 +\\ Human} \\ \hline
\makecell{$n=6$ \\ $k=5$} & \underline{192} & \underline{192} & \underline{192} & \underline{192} & \underline{192} \\ \hline
\makecell{$n=10$ \\ $k=8$} & 5114 & 5100 & 5112 & 5114 & \underline{5120} \\ \hline
\makecell{$n=15$ \\ $k=12$} & 245252 & 240509 & 238215 & 245628 & 245748 \\ \hline

\end{tabular}

% \begin{table}
% \centering
% \begin{tabular}{|l|c|c|c|c|c|}
% \hline
%  & \makecell{Tabu \\search} & \makecell{OpenEvolve} & \makecell{Composition-\\constrained \\ OpenEvolve} & \makecell{CPro1} & \makecell{CPro1 +\\ Human} \\ \hline
% $n=6$ $k=5$ & \underline{192} & \underline{192} & \underline{192} & \underline{192} & \underline{192} \\ \hline
% $n=10$ $k=8$ & 5114 & 5100 & 5112 & 5114 & \underline{5120} \\ \hline
% $n=15$ $k=12$ & 245252 & 240509 & 238215 & 245628 & 245748 \\ \hline

% \end{tabular}

\label{tab:alg-comparison}
\end{table}

% Alternate version where n and k are on the same line
% \newcommand{\nk}[2]{\makebox[3em][l]{$n=#1$}\,\,$k=#2$}
% \begin{table}
% \centering
% \begin{tabular}{|l|c|c|c|c|c|}
% \hline
%  & \makecell{Tabu \\ search}
%  & \makecell{OpenEvolve}
%  & \makecell{Composition-\\constrained \\ OpenEvolve}
%  & \makecell{CPro1}
%  & \makecell{CPro1 +\\ Human} \\ \hline

% \nk{6}{5}   & \underline{192} & \underline{192} & \underline{192} & \underline{192} & \underline{192} \\ \hline
% \nk{10}{8}  & 5114 & 5100 & 5112 & 5114 & \underline{5120} \\ \hline
% \nk{15}{12} & 245252 & 240509 & 238215 & 245628 & 245748 \\ \hline

% \end{tabular}
% \caption{Comparison of the number of edges sliced by our manually created tabu search algorithm, AI-only systems, and AI systems with human intervention. Underlines indicate full hypercube slicing.} 
% \label{tab:alg-comparison}
% \end{table}
% \end{table}

%% file: sections-arxiv/conclusion.tex
\section{Discussion}

This work illustrates the effectiveness of AI in finding constructions in large high-dimensional spaces which can prove challenging for effective search by human researchers~\cite{swirszcz2025advancing,peterson2021using}. Our new constructions were obtained with open source tools on a personal machine without requiring excessive computing power, pointing to the applicability of AI in mathematics (and other disciplines) for researchers that may not have access to computational resources of industrial scale, so long as LLM queries remain within budget. In particular, we believe that CPro1 could be useful in finding constructions in combinatorics, geometry and theoretical computer science. 
%The geometric nature of our problem adds to the evidence of the utility of AI in finding geometric constructions~\cite{berthold2026global,berczi2026flow,swirszcz2025advancing} and 
%We believe that CPro1 could be useful in finding constructions in combinatorics, geometry and theoretical computer science. 

Importantly, while AI-driven search and discovery can explore vast spaces that are otherwise intractable to attempt through human effort alone, we also advocate that large-scale AI-driven search methods be explicitly designed to support human-in-the-loop workflows. In our experience, the task was not accomplished by an AI system in isolation; rather, progress emerged when the automated, high-throughput search surfaced human-interpretable patterns and partial insights, which were then used to manually guide further refinement. We therefore argue that AI-driven automated discovery systems should actively organize and present observations and recurring patterns in ways that are accessible to human researchers, enabling targeted interventions and closer integration with the search and discovery system.

% Importantly, while AI-driven search and discovery [X], we advocate for ...
% [Something about: we advocate for AI-driven search methods which allow for human-in-the-loop workflows... sort of emphasizing that AI systems weren't able to accomplish this task alone, but if a workflow was built that by-default surfaced human-interpretable insights gathered by an automated AI search, which could then be used for further refinement, that would be ideal...]
% %^ above should probably be specifically about large-scale search
%TODO: https://www.tandfonline.com/doi/full/10.1080/0960085X.2025.2475962#d1e191 seems very relevant to cite for the story, though that paper is more about decision-making rather than mathematical discovery
%https://arxiv.org/abs/2306.01694

We expect that further improvements on the upper bound for $S(n)$ are possible. For instance, when searching for lower bounds for $S(15, 12)$ we find constructions of $12$ hyperplanes only $12$ edges from fully slicing $Q_{15}$, constructions for $S(14, 11)$ only $22$ edges from slicing $Q_{14}$, and constructions for $S(13,10)$, $S(12,9)$, and $S(11, 8)$ only $24$ edges from slicing their respective hypercubes.
Considering the belief that $S(n)=\Omega(n)$, the progress recorded here could bring us closer to finding the precise value of $S(n)$.
Beyond mathematical interest in slicing the hypercube, the problem of lower bounding $S(n,k)$ is attractive for evaluating AI for math systems considering the ease of verifying that a given solution slices a given set of edges coupled with the challenges in constructing such solutions. 

A recurring feature in our constructions is that the collections of slicing hyperplanes appear to be low-dimensional in nature: the first several coefficients of each hyperplane are identical, and we observe additional forms of repetitive structure in many solutions (\Cref{appendix:additional-constructions-and-observations}). %(\Cref{SI-appendix:constructions_n10k8}).
These patterns may point to underlying structural properties of optimal or near-optimal slicing hyperplanes. An interesting direction of further study is to formalize the existence of ``good" solutions to the slicing problem carrying low-dimensional structure, and to what extent full solutions admit such properties. 

Our work has several limitations. Although finite constructions of slicing hyperplanes translate to an upper bound on $S(n)$, they fall short of proving upper bounds on $S(n)$ of order $o(n)$ (if such upper bounds exist).
Furthermore, our methodology is not applicable to proving lower bounds on $S(n)$ or upper bounds on $S(n,k)$. Methods from the field of formal verification and SAT solvers could be potentially useful for finding better upper bounds for $S(n,k)$ as well as better lower bounds for $S(n)$ in the regime of $n,k$ studied here (e.g., $n,k \leq 15$). Discovering tighter lower and upper bounds on $S(n,k)$ using a combination of automated tools such as CPro1 and formal methods is an exciting future research direction. 

%% file: sections-arxiv/acknowledgements.tex
\section*{Acknowledgements}
The authors are grateful to Adam Zsolt Wagner for his collaboration in the early phase of the project (August-November 2024) and creating the initial implementation of the tabu search algorithm which was later developed into the version in \Cref{alg:tabu_search}. He declined our offer to be a co-author, as he did not feel his contribution granted being a co-author of this paper.
This research was performed using computational resources supported by the Academic \& Research Computing group at Worcester Polytechnic Institute.
The authors thank Noga Alon and Swarat Chaudhuri for their useful feedback.

%% file: sections-arxiv/appendix.tex
\crefalias{section}{appendix}
\crefalias{subsection}{appendix}

\section{Additional Constructions and Observations on Structure}\label{appendix:additional-constructions-and-observations}

\subsection{Paterson construction for slicing \texorpdfstring{$Q_{6}$}{Q6}}
Solutions are listed such that each row corresponds to the vector of coefficients of a hyperplane, with the final entry representing the bias term.  For example, the first hyperplane is $1x_1+1x_2+1x_3+3x_4+3x_5-4x_6=0$.

\begin{equation}
\label{eq:56planes}
\begin{array}{rrrrrrr}
1 & 1 & 1 & 3 & 3 & -4 & 0 \\
-2 & -2 & -2 & 3 & 3 & -1  & 0 \\
3 & 3 & 3 & 1 & 1 & -4 & 0 \\
-1 & -1 & -1 & 3 & 3 & 6 & 0 \\
3 & 3 & 3 & 1 & 1 & 8 & 0 \\

\end{array}
\end{equation}

% TODO: Remove section numbers (but need to keep name ref?)
\subsection{Additional constructions for slicing \texorpdfstring{$Q_{10}$}{Q10}}\label{appendix:constructions_n10k8}
Additional constructions which slice $Q_{10}$ are listed below. Note that there exist many solutions without $-9$ as the repeated coefficient, we simply find that fixing this value leads to finding solutions quickly. We refer to the values in the columns that are not fixed to this coefficient as the \emph{free values}.

Below is the first full solution to $Q_{10}$ that we found.

\begin{equation*}
\label{eq:108planes_orig}
\begin{array}{rrrrrrrrrrr}
-9 & -9 & -9 & -9 & -9 & -9 &   7 & -16 &   5 &  35 & 0.5 \\
-9 & -9 & -9 & -9 & -9 & -9 & -32 &  -4 & -17 &   8 & 0.5 \\
-9 & -9 & -9 & -9 & -9 & -9 &  32 &   5 &  19 &  -4 & 0.5 \\
-9 & -9 & -9 & -9 & -9 & -9 &  -3 &  15 &  -3 & -38 & 0.5 \\
-9 & -9 & -9 & -9 & -9 & -9 &  15 &   3 & -36 &   4 & 0.5 \\
-9 & -9 & -9 & -9 & -9 & -9 &   8 & -35 &  -2 & -12 & 0.5 \\
-9 & -9 & -9 & -9 & -9 & -9 &  -4 &  33 &   7 &  16 & 0.5 \\
-9 & -9 & -9 & -9 & -9 & -9 & -18 &  -4 &  34 &  -5 & 0.5 \\
\end{array}
\end{equation*}

Upon closer observation of this solution, it appears that the planes can be paired such that the free values in one member of the pair are roughly equal to the negation of the free values in the other member of the pair, after permutation. Indeed, enforcing an equality constraint of this form during the search enables it to find solutions for $k=8$ $n=10$ and $k=10$ $n=12$, and the search is more efficient due to the greatly reduced size of the search space. However, we fail to find solutions of this form even for $k=6$ $n=7$ despite the fact it is fast to find solutions without these constraints, and most partial solutions obtained in higher dimensions under these constraints are also worse. This suggests this structure is not a global phenomenon across all $n$ and $k$, it is interesting to ask under what conditions such solutions are possible.

We also observed many solutions where the free values in several planes were (approximately) a permutation of the free values in another plane, with some values negated, thus forming groups of multiple `similar' planes. Imposing these constraints leads to highly structured solutions for $n=10$ $k=8$ but does not appear to work (that is, our algorithm could not find full slicing solutions) for most other values of $k$ and $n$. The solution presented in \Cref{main-theorem} of the primary text was obtained using this method, it can be observed that the first 4 rows each use the same free values but permuted or negated. The same is true for the final 4 rows, using the same pattern.

Additionally, if we restrict the algorithm to search over only solutions in which, across all planes, the number of distinct coefficient magnitudes is at most a given threshold, we also arrive at full solutions with high degrees of structure, with one such construction presented below.

\begin{equation*}
\label{eq:108alt1}
\begin{array}{rrrrrrrrrrr}
-9 & -9 & -9 & -9 & -9 & -9 & 30 & 4 & 3 & -20 & 0.5 \\
-9 & -9 & -9 & -9 & -9 & -9 & 20 & -3 & 4 & 30 & 0.5 \\
-9 & -9 & -9 & -9 & -9 & -9 & -30 & -3 & -4 & 20 & 0.5 \\
-9 & -9 & -9 & -9 & -9 & -9 & -20 & 3 & -3 & -30 & 0.5 \\
-9 & -9 & -9 & -9 & -9 & -9 & -3 & -38 & 11 & -3 & 0.5 \\
-9 & -9 & -9 & -9 & -9 & -9 & 4 & -11 & -38 & -4 & 0.5 \\
-9 & -9 & -9 & -9 & -9 & -9 & 3 & 38 & -11 & 3 & 0.5 \\
-9 & -9 & -9 & -9 & -9 & -9 & -4 & 11 & 38 & 3 & 0.5 \\
\end{array}
\end{equation*}
\medskip

\subsection{CPro1's initial improvement to S(15,12)}
\phantomsection\label{appendix:constructions}
This CPro1-produced set of 12 hyperplanes slices 245628 out of the 245760 edges of $Q_{15}$, an improvement over our prior tabu search result of 245252.  CPro1 produced this solution which use the same value in the initial columns, without being prompted or constrained to do so.

\begin{equation*}
\label{eq:152}
\begin{array}{rrrrrrrrrrrrrrrr}
-1  &   -1  &   -1  &   -1  &   -1  &   -1  &   -1  &   -1  &   -1  &   -1  &   -1  &   -1  &   -1  &   -2  &   -2  &   0.5 \\
-1  &   -1  &   -1  &   -1  &   -1  &   -1  &   -1  &   -1  &   -1  &   -1  &   -1  &   -1  &   -1  &   -2  &   0  &    0.5 \\
-1  &   -1  &   -1  &   -1  &   -1  &   -1  &   -1  &   -1  &   -1  &   -1  &   -1  &   -1  &   -1  &   2  &    -2  &   0.5 \\
-1  &   -1  &   -1  &   -1  &   -1  &   -1  &   -1  &   -1  &   -1  &   -1  &   -1  &   -1  &   -1  &   2  &    0  &    0.5 \\
-1  &   -1  &   -1  &   -1  &   -1  &   -1  &   -1  &   -1  &   -1  &   -1  &   -1  &   -1  &   -1  &   -6  &   -2  &   0.5 \\
-1  &   -1  &   -1  &   -1  &   -1  &   -1  &   -1  &   -1  &   -1  &   -1  &   -1  &   -1  &   -1  &   -6  &   10  &   0.5 \\
-1  &   -1  &   -1  &   -1  &   -1  &   -1  &   -1  &   -1  &   -1  &   -1  &   -1  &   -1  &   -1  &   -6  &   0  &    0.5 \\
-1  &   -1  &   -1  &   -1  &   -1  &   -1  &   -1  &   -1  &   -1  &   -1  &   -1  &   -1  &   -1  &   6  &    0  &    0.5 \\
-1  &   -1  &   -1  &   -1  &   -1  &   -1  &   -1  &   -1  &   -1  &   -1  &   -1  &   -1  &   -1  &   8  &    16  &   0.5 \\
-1  &   -1  &   -1  &   -1  &   -1  &   -1  &   -1  &   -1  &   -1  &   -1  &   -1  &   -1  &   -1  &   8  &    0  &    0.5 \\
-1  &   -1  &   -1  &   -1  &   -1  &   -1  &   -1  &   -1  &   -1  &   -1  &   -1  &   -1  &   -1  &   -10  &  0  &    0.5 \\
-1  &   -1  &   -1  &   -1  &   -1  &   -1  &   -1  &   -1  &   -1  &   -1  &   -1  &   -1  &   -1  &   10  &   0  &    0.5 \\
\end{array}
\end{equation*}
\medskip

\FloatBarrier
\section{The Reduced Hypercube}\label{appendix:reduced-hypercube}
The sets of hyperplanes slicing the highest number of edges exhibited a common property: each hyperplane had groups of identical coefficients, obeying a composition.
\begin{definition}
A composition $\beta$ of $Q_n$ is an ordered list of positive integers $[b_1, b_2, \ldots, b_\ell]$ such that $\sum_{i=1}^{\ell} b_i = n$, where each $b_i$ represents the size of the $i$-th group. Let $B_j$ denote the the $j$'th group of $\beta$.
\end{definition}
For example, the solution from \Cref{main-theorem} %Equation \ref{main-eq:108planes} 
of the main text satisfies a $[6,1, 1, 1, 1]$ composition, while Paterson's solution in Equation \ref{eq:56planes} satisfies $[3, 2, 1]$.

\begin{figure}[th]
    \centering
    \includegraphics[width=0.48\textwidth]{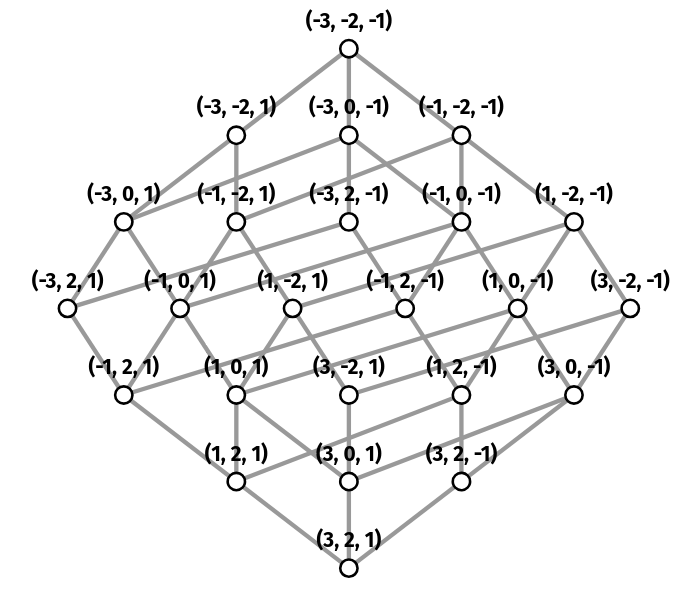}
    \hfill
    \includegraphics[width=0.48\textwidth]{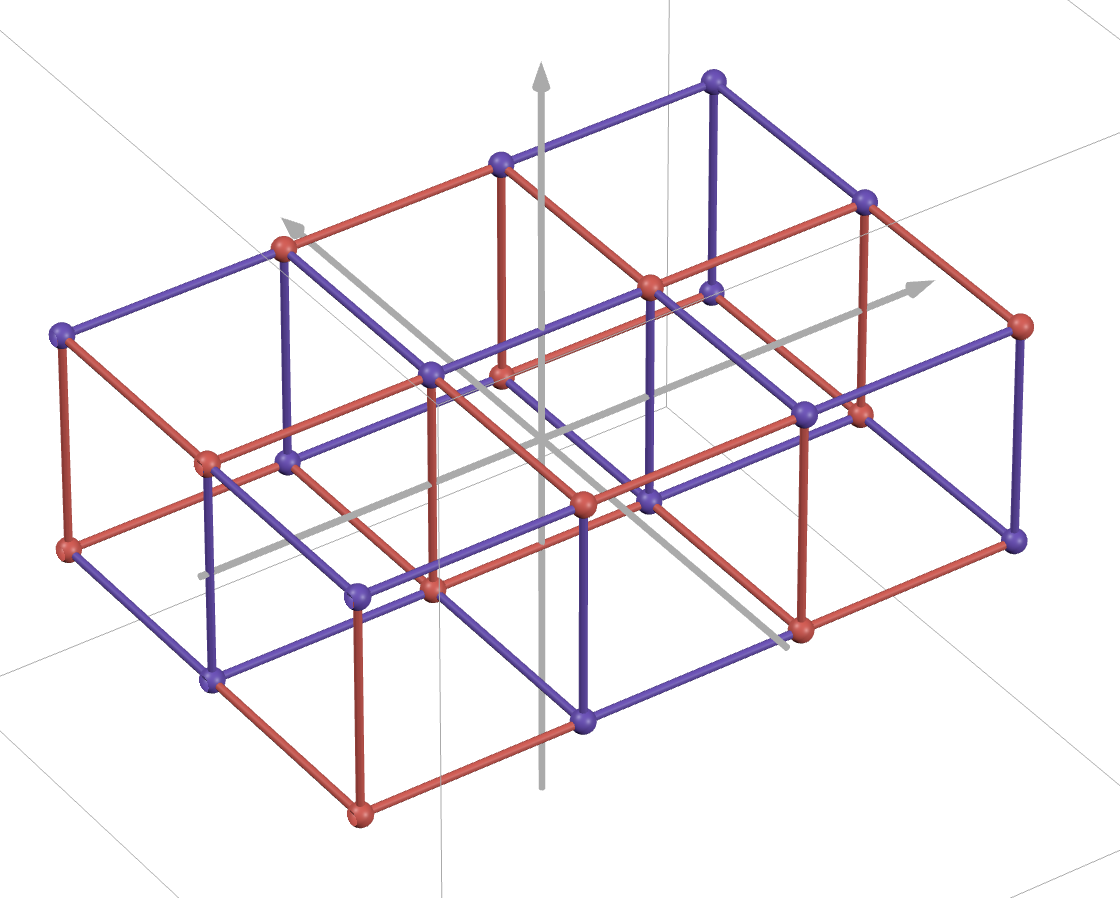}

    \caption{The $6$-dimensional hypercube $Q_6$ under the composition $\beta = [3,2,1]$ is transformed into a $3\times2\times1$ lattice in three dimensions with $46$ edges and $24$ vertices.}
    \label{fig:6d-hypercube-combined}
\end{figure}

\begin{definition}
\label{def:reduced-hyperplane}
A hyperplane $\mathbf{a}$ that satisfies a composition $\beta$ with $\ell$ groups has a reduced hyperplane representation, $\mathbf{a}^\beta = (a^\beta_1,\dots,a^\beta_\ell)\in\mathbb{R}^\ell$ where $a_i = a^\beta_j$ for all $i\in B_j$.
\end{definition}
For example, if $\mathbf{a}$ has coefficients $(4, 4, 5, 5, 1)$, it satisfies composition $\beta=[2,2,1]$ and has reduced representation $\mathbf{a}^\beta = (4,5,1)$.

\begin{definition}
\label{def:reduced-vertex}
Any vertex $\mathbf{v} = (v_1, v_2,\cdots, v_n) \in Q_n$ has a \textit{reduced vertex} representation $\mathbf{v}^\beta = (v^\beta_1, v^\beta_2,\cdots, v^\beta_\ell)$ under a composition $\beta$ in $\ell$ dimensions where $v^\beta_j = \sum_{i \in \beta_j} v_i$.
\end{definition}

Observe that the possible values for $v^\beta_j$ are $\{-{b}_j, -{b}_j + 2, \cdots,  {b}_j - 2 ,{b}_j\}$: The sum of $b_j$ vertex coordinates (with value $\pm 1$) has minimum $-{b}_j$ and maximum ${b}_j$. Intermediate values are obtained by changing a single coordinate from $-1$ to $1$, which is a difference of $2$.

\begin{definition}\label{def:reduced-edge}
    Two reduced vertices, $\mathbf{v}^\beta$ and $\mathbf{w}^\beta$, are connected by a \textit{reduced edge}, $e^\beta = (\mathbf{v}^\beta, \mathbf{w}^\beta)$, if the edge $e = (\mathbf{v}, \mathbf{w})$ exists on the original hypercube.
\end{definition}

\begin{theorem}
    Two reduced vertices, $\mathbf{v}^\beta$ and $\mathbf{w}^\beta$, are connected by an edge if and only if they differ by exactly 2 in one coordinate. 
\end{theorem}
\begin{proof}
$\mathbf{v}^\beta$ and $\mathbf{w}^\beta$ are connected by an edge if and only if $\mathbf{v}$ and $\mathbf{w}$ are adjacent. Without loss of generality, assume they differ in the $i$-th coordinate. Observe that $|v_i - w_i| = 2$ as one must be $-1$ and the other $1$, and no other coordinate differs.
\begin{equation*}
\begin{split}
\left|v^\beta_{j}  - w^\beta_{j} \right| &=  \left|  \sum_{k\in {B}_j}v_{k} -  \sum_{k\in {B}_j}w_{k} \right|\\
&=  \left|  \left(\sum_{k\in {B}_j\setminus \{i\}} v_k\right) + v_i -  \left(\sum_{k\in {B}_j\setminus \{i\}}w_k\right) - w_i \right|\\
& = \left|v_i  - w_i \right|  = 2
\end{split}    
\end{equation*}
\end{proof} 
\begin{corollary}\label{cor:total-reduced-sum}
    The sum of coordinates of any two reduced vertices connected by an edge differ by exactly 2.
\end{corollary}

\begin{definition}
    A reduced hypercube, $Q^{\beta}=(V^\beta, E^\beta)$, is the reduced representation of a hypercube $Q_n$ given a composition $\beta$, where $V^\beta$ is the set of all reduced vertices, and $E^\beta$ is a set of reduced edges between these vertices.
\end{definition}

We can easily find $|V^\beta|$ by applying Definition \ref{def:reduced-vertex} and the observation that follows:
\begin{equation}\label{eq:reduced-vertex-cardinality}
|V^\beta| = \prod_{i=1}^\ell |\{-{b}_i, -{b}_i+ 2, \cdots, {b}_i - 2, {b}_i\}|   = \prod_{i=1}^\ell ({b}_i+1).
\end{equation}

Edges (in $Q^{\beta}$) in direction (dimension) $i$ connect consecutive values in the $i$-th coordinate. There are $b_i$ such transitions (from $-b_i$ to $-b_i+2$, from $-b_i+2$ to $-b_i+4$, etc.). For each of the $\prod_{j \ne i}(b_j + 1)$ choices of the other coordinates, we get $b_i$ edges. Summing over all $\ell$ directions:
\[
|E^\beta| = \sum_{i=1}^\ell b_i \prod_{j \ne i}(b_j + 1).
\]
To obtain an alternative form, factor out $|V^\beta| = \prod_{j=1}^\ell(b_j+1)$ to arrive at $|V^\beta|(\ell-\sum_{i=1}^\ell\frac{1}{{b}_i +1})$.

\begin{theorem}\label{thm:reduced-dot}
    For a hyperplane $\mathbf{a}$ and vertex $\mathbf{v} \in Q_n$ satisfying a composition $\beta$, $\mathbf{a}^{\beta} \cdot \mathbf{v}^{\beta} = \mathbf{a} \cdot \mathbf{v}$
\end{theorem}

\begin{proof}
\[ \mathbf{a} \cdot \mathbf{v} = \sum_{k=1}^n a_k v_k + b = \sum_{i=1}^\ell a^\beta_i \left(\sum_{j \in B_i} v_j\right) + b = \sum_{i=1}^\ell a^\beta_i \cdot v^\beta_i + b = \mathbf{a}^{\beta} \cdot \mathbf{v}^{\beta}
\]
\end{proof}

\begin{theorem}\label{thm:reduced-sliced-equiv}
A set of hyperplanes $A$ all satisfying $\beta$ slice all the edges of the hypercube $Q$ if and only if the set of reduced hyperplanes $A^\beta$ slices all the edges of the reduced hypercube $Q^\beta$.
\end{theorem}
\begin{proof}
By Definition \ref{def:reduced-edge}, every edge of $Q_n$ maps to some edge in $E^\beta$ and every edge in $E^\beta$ is the image of at least one edge in $Q_n$. The result follows from Theorem \ref{thm:reduced-dot}.
\end{proof}

Theorem \ref{thm:reduced-sliced-equiv} establishes that the search for slicing sets can be restricted to the reduced hypercube. The reduced hypercube preserves the property that vertices of any edge differ in exactly one coordinate, allowing us to apply existing search techniques without modification.

\subsection*{Computing the Sliced Edges in the Reduced Hypercube}
%Given a composition and a set of sliced edges in the reduced hypercube
To map the number of edges sliced by a hyperplane on the reduced hypercube to the original hypercube, we need to know the \textit{number} of edges corresponding to a reduced edge. Let $\mathbf{\widehat v}^\beta$ represent the number of negative coordinates in each composition group, where
\begin{equation*}
    \widehat v^\beta_j = \sum_{i\in {b}_j}[v_i =-1].
\end{equation*}
Note that this representation still uniquely identifies each reduced vertex, as the number of $-1$s directly gives the sum.

Each reduced vertex represents several vertices in the hypercube. For a composition $\beta$ we can calculate the number of vertices which map to a reduced vertex, $\mathbf{v}^\beta$, by considering all vertices in which the sum of the coordinates in each composition group are equal. If the sums are equal, the number of negative coordinates is also equal, and if a coordinate is not negative it is positive. Therefore, we only need to count the number of ways of selecting negative coordinates in each composition group:
\begin{equation}\label{eq:num-reduced-vertices}
    |\mathbf{v}^\beta| = \prod_{i=1}^{\ell} \binom{{b}_i}{\widehat v^\beta_i}.
\end{equation}

Each reduced edge also corresponds to several edges in the hypercube, as several pairs of vertices on the hypercube will map to the same pair of reduced vertices.

\begin{definition}
For a reduced edge $e = (u^\beta, v^\beta) \in E^\beta$, its \emph{multiplicity} $\mu(e)$ is the number of edges in $Q_n$ that map to $e$. 
\end{definition}

For a reduced vertex $v^\beta$, let $\widehat{v}^\beta_j = (b_j - v^\beta_j)/2$ denote the number of $-1$ coordinates in block $j$ among its preimages.

\begin{theorem}
Let $e = (u^\beta, v^\beta) \in E^\beta$ with $u^\beta_i < v^\beta_i$. Then
\[
\mu(e) = \widehat{u}^\beta_i \prod_{j=1}^\ell \binom{b_j}{\widehat{u}^\beta_j}.
\]
\end{theorem}

\begin{proof}
An original edge maps to $e$ iff it flips one coordinate in block $i$ from $-1$ to $+1$. The number of preimages of $u^\beta$ is $\prod_j \binom{b_j}{\widehat{u}^\beta_j}$ (choosing which coordinates are $-1$ in each block). Each preimage has $\widehat{u}^\beta_i$ such coordinates to flip.
\end{proof}

\section{Our Tabu Search-Inspired Algorithm}
% \dr{Give the algorithm a name}
Prior to creating the final edge-weighted search, we experimented extensively with a tabu search-inspired algorithm (\Cref{alg:tabu_search}). Details for this algorithm are presented here.
Let $\mathcal S$ denote the set of feasible solutions, that is, hyperplanes with integer coefficients obeying the given composition (if any).
For each $S \in \mathcal S$, let $\phi(S)$ denote the set of edges cut by $S$ in the (reduced) hypercube.
We define a preorder $\succeq$ on $\mathcal S$ by
\[
S \succeq S' \iff |\phi(S)| \ge |\phi(S')|.
\]
Two solutions $S$ and $S'$ are considered equivalent under this order if $|\phi(S)|=|\phi(S')|$. Additionally, we define $\phi(\emptyset) =\emptyset$.

\begin{algorithm}[htb]
\caption{Our version of tabu search}\label{alg:tabu_search}
\begin{algorithmic}[1]
\State $S^* \gets \emptyset$
\While{time limit not exceeded}
\State $S \gets S_0$ \Comment{Random starting solution}
\State $E \gets \{h(S_0)\}$ \Comment{Set of all seen solutions, hashed}
\State $N \gets \{S_0$\}  \Comment{Ordered set of the best unexplored solutions}
\State $e \gets |E|$

\Statex{\hspace*{\algorithmicindent}Terminate when more than $R$ solutions are explored without improvement}
\While{$|E| - e < R$}
    \State $\widehat{S} \gets \arg\max_{S' \in N}|\phi(S')|$ \Comment{Get best unexplored solution}
    \State $N \gets N \setminus \{\widehat{S}\}$ \Comment{Remove solution from unexplored set}

    \For{$S' \in \text{Ham}_1(\widehat{S})$} \Comment{Iterate neighbors of Hamming distance 1}
        \If{$h(S') \notin E$}
            \State $E \gets E \cup \{h(S')\}$ \Comment{Mark as seen}
            \State $N \gets N \cup \{S'\}$ \Comment{Add to unexplored solutions}

            \If{$S' \succ S^*$}
                \State $S \gets S'$ \Comment{Track best solution}
                \State $e \gets |E|$
            \EndIf
        \EndIf
    \EndFor
\EndWhile
\If{$S \succ S^*$}
    \State $S^* \gets S$ \Comment{Track overall best solution}
\EndIf
\EndWhile

\end{algorithmic}
\end{algorithm}

We constrain our search to solutions with integer-valued planes with coefficients in $[-c, c]$, where $c$ is a hyperparameter. Further, we find that restricting the constant terms for all planes to $0$ (plus a small fractional offset to avoid passing through vertices) improves the search.

% hash might not be the best word for this...
To hash a solution, we fix an ordering of the edges of the (reduced) hypercube.
Given a solution $S$, each plane in $S$ is encoded as a binary vector whose $i$th
entry is $1$ if the plane slices the $i$th edge and $0$ otherwise.
Stacking these vectors yields a binary-valued matrix representation $h(S)$ of $S$,
which we store in the tabu set in place of $S$.
Empirically, using this representation instead of storing solutions directly
significantly improves search performance, as it prevents many effectively
equivalent solutions from being treated as distinct and unnecessarily explored.

We additionally find that searching neighbors only over a local Hamming distance, that is, searching neighbor plane sets where exactly one coefficient differs by not more than $d$ (a hyperparameter), is an important improvement to the tabu search algorithm that significantly reduces the time it takes to find solutions. We hypothesize that this is because, after several iterations of greedy refinement, large magnitude moves in coordinate space are exceedingly unlikely to be productive. Consequently, restricting the search to a smaller local neighborhood avoids unnecessary computation.

Random restarts are essential to avoid escaping local optima. The overall algorithm is inspired by \cite{mehrabian2023finding}.

% TODO: mention reduced hypercube, NOT forcing all planes to have same first columns

\FloatBarrier
\section{Simple Verification Algorithm}\label{appendix:verification-algorithm}
The following algorithm provides a non-optimized but simple way of verifying whether a collection of planes slice all edges of $Q_n$. A simple, commented version of this algorithm is provided at \url{github.com/DSoiffer/upper-bounds-for-hypercube-slicing}. A web-based version for inspecting and verifying solutions more closely is provided at \url{hypercube-slicing.pages.dev}.
\begin{algorithm}[h]
\caption{Simple solution verification}
\begin{algorithmic}[1]
\Require Hyperplanes $\mathcal{H}=\{(a_j,b_j)\}_{j=1}^k$
\For{$i=1$ to $n$}
    \For{each $x\in\{-1,1\}^n$ with $x_i=-1$}
        \State $x^{(i)} \gets (x_1,\dots,x_{i-1},-x_i,x_{i+1},\dots,x_n)$
        \State sliced $\gets$ \textbf{false}
        \For{$j=1$ to $k$}
            \If{$(\langle a_j,x\rangle-b_j)(\langle a_j,x^{(i)}\rangle-b_j) < 0$}
                \State sliced $\gets$ \textbf{true}
                \State \textbf{break}
            \EndIf
        \EndFor
        \If{\textbf{not} sliced}
            \State \Return \textbf{false}
        \EndIf
    \EndFor
\EndFor
\State \Return \textbf{true}
\end{algorithmic}
\end{algorithm}

\FloatBarrier
\section{Additional Tables and Prior Bounds}
Previous work has searched for the value of of $S(n)$ and $S(n,k)$ for $n\leq8$, including an exhaustive search for $n\leq6$ \cite{newbounds, emamy2008coverings}. Mike Paterson showed $S(n) \leq \lceil\frac{5n}{6} \rceil$ with the construction in Equation \ref{eq:56planes} \cite{saks1993slicing}. \Cref{tab:previous-bounds} summarizes these results.

\begin{table}[h]
\centering
\caption{Best known lower bounds on $S(n,k)$ prior to our work.}\label{tab:previous-bounds}
\begin{tabular}{|c|cccccccc|}
\hline
$S(n,k)$ & $k=1$ & $k=2$ & $k=3$ & $k=4$ & $k=5$ & $k=6$ & $k=7$ & $k=8$ \\
\hline
$n=3$ & 6 & 10 & 12 & -- & -- & -- & -- & --  \\
$n=4$ & 12 & 24 & 30 & 32 &  & -- & -- & --  \\
$n=5$ & 30 & 54 & 70 & 78 & 80 & -- & -- & -- \\
$n=6$ & 60 & 120 & 160 & 184 & 192 & 192 & -- & --  \\
$n=7$ & 140 & 260 & 350 & 410 & 434 & 448 & 448 & --  \\
$n=8$ & 280 & 560 & 770 & 908 & 980 & 1008 & 1024 & 1024 \\
\hline
\end{tabular}
\end{table}

We also present the previous best values achieved by our manually constructed tabu search in \Cref{appendix:tab-tabu-bounds}. Note that these are \emph{not} our best results, but are presented for sake of comparison.

\begin{table}[h]
\centering
\caption{Lower bounds on $S(n,k)$ discovered by tabu search. Each bounds follows from a set of $k$ hyperplanes slicing the reported number of edges in $Q_n$. Bold numbers indicate improvements over previous best-known values, underlined numbers indicates the maximum possible value (full slicing).}\label{appendix:tab-tabu-bounds}
\begin{tabular}{|c|cccccccc|}
\hline
$S(n,k)$ & $k=4$ & $k=5$ & $k=6$ & $k=7$ & $k=8$ & $k=9$ & $k=10$ & $k=11$ \\
\hline
$n=5$  & \underline{78}   & \underline{80}   & --   & --   & --   & --   & --    & --    \\
$n=6$  & 184  & \underline{192}  & \underline{192}  & --   & --   & --   & --    & --    \\
$n=7$  & 410  & \textbf{440}  & \underline{448}  & \underline{448}  & --   & --   & --    & --    \\
$n=8$  & \textbf{920}  & 980  & \textbf{1016} & \underline{1024} & \underline{1024} & --   & --    & --    \\
$n=9$  & \textbf{1974} & \textbf{2184} & \textbf{2254} & \textbf{2298} & \underline{2304} & \underline{2304} & --    & --    \\
$n=10$ & \textbf{4312} & \textbf{4704} & \textbf{4984} & \textbf{5064} & \textbf{5114} & \underline{5120} & \underline{5120}  & --    \\
$n=11$ & \textbf{9072} & \textbf{10052} & \textbf{10536} & \textbf{10844} & \textbf{11042} & \textbf{11258} & \underline{11264} & \underline{11264} \\
\hline
\end{tabular}
\end{table}

\begin{table}[h]
\centering
\caption{Additional values for larger $(n,k)$ discovered by tabu search.}
\begin{tabular}{|c|c|c|l|}
\hline
$n$ & $k$ & Best found & Max edges \\
\hline
13 & 10 & \textbf{53008}   & \underline{53248} \\
14 & 11 & \textbf{114286} & \underline{114688} \\
15 & 12 & \textbf{245252}  & \underline{245760} \\
16 & 13 & \textbf{523430}  & \underline{524288} \\
17 & 14 & \textbf{1114088} & \underline{1114112} \\
\hline
\end{tabular}
\end{table}